\documentclass[11pt]{article}

\usepackage[utf8]{inputenc}
\usepackage[T1]{fontenc}
\usepackage{lmodern}
\usepackage[margin=1in]{geometry}
\usepackage{microtype}
\usepackage{amsmath,amssymb,amsthm,mathtools}
\usepackage{booktabs,array,tabularx}
\usepackage{xcolor}
\usepackage{listings}
\usepackage{tikz}
\usetikzlibrary{positioning,arrows.meta}
\usepackage[numbers,sort&compress]{natbib}
\usepackage[colorlinks=true,linkcolor=blue!50!black,citecolor=blue!50!black,urlcolor=blue!50!black]{hyperref}

\hypersetup{
  pdftitle={Masked Attention Pooling Is Invariant to Padding and Permutation: A Formal Proof in Lean},
  pdfauthor={Yusuf Khaled Sadek, Enes Yilmaz},
  pdfkeywords={multiple instance learning, attention pooling, padding, permutation invariance, Lean, formal verification, JAX}
}

\newtheorem{theorem}{Theorem}
\newtheorem{proposition}[theorem]{Proposition}
\newtheorem{corollary}[theorem]{Corollary}

\theoremstyle{definition}
\newtheorem{definition}[theorem]{Definition}
\theoremstyle{remark}

\newcommand{\R}{\mathbb{R}}
\newcommand{\Pool}{\operatorname{Pool}}
\newcommand{\lean}[1]{\texttt{\detokenize{#1}}}
\definecolor{codebg}{gray}{0.96}
\definecolor{kw}{rgb}{0.10,0.25,0.55}
\definecolor{cm}{gray}{0.40}

\lstdefinelanguage{Lean}{
  morekeywords={theorem,def,noncomputable,fun,let,if,then,else,by,import,
    namespace,end,open,scoped,with,have,intro,cases,funext,simp,rfl,rw,exact,
    unfold,apply,change,only,match},
  sensitive=true,
  morecomment=[l]{--},
  morecomment=[s]{/-}{-/},
  morestring=[b]",
}
\title{Masked Attention Pooling Is Invariant to\\
Padding and Permutation: A Formal Proof in Lean}
\author{Yusuf Khaled Sadek\qquad Enes Yilmaz\\[2pt]
\small The Ohio State University, Columbus, OH, USA\\[2pt]
\small\href{mailto:yusufkhsadek@gmail.com}{\texttt{yusufkhsadek@gmail.com}}}
\date{October 2026}

\begin{document}
\maketitle

\begin{abstract}
Attention-based multiple instance learning (MIL) pools a variable-size bag of
instance embeddings into one vector. In batched training, bags are padded to a
common length, and the pooled output is expected not to depend on how much
padding was added or on the order in which instances are stored. This
expectation holds only if padded rows are removed from both the numerator and
the softmax normalizer. Common implementation idioms meet it only under side
conditions. Zeroing padded features alone does not work. Masking logits with
$-\infty$ fails on empty bags and on nonfinite padding. Masking with a large
negative constant is exact only after floating-point underflow. We define masked attention pooling with a
guarded normalizer and prove in Lean~4 with mathlib that it is invariant under
any bijective reindexing of rows and under appending any finite number of
masked rows with arbitrary real features and logits. We also prove that
values at masked positions are irrelevant and that an all-masked bag maps to
zero. The theorems hold for arbitrary finite index types and feature
dimensions. Using the \texttt{jax-lean} certificate pipeline, we connect these
theorems to the Jaxpr traced from the JAX kernel itself. Lean checks
translation certificates for two concrete shapes, identifies the generated
functions with the generic specification, and transfers both invariances to
the imported Jaxpr's real-arithmetic semantics. All audited theorems depend
only on Lean's three standard axioms. The results concern ideal real
arithmetic. Floating-point execution, XLA, and the JAX tracer and importer
remain outside the proof. The development is short and publicly available.
\end{abstract}

\section{Introduction}

In multiple instance learning~\cite{dietterich1997}, a label is attached to a
\emph{bag} of instances rather than to each instance. Attention-based MIL
pooling~\cite{ilse2018} computes a scalar score for every instance, normalizes
the scores with a softmax over the bag, and returns the weighted average of the
instance embeddings. This construction underlies much of weakly supervised
computational pathology, where a whole-slide image is a bag of thousands of
tile embeddings and only the slide carries a label~\cite{campanella2019,lu2021}.
Ilse et al.\ introduce the operator as a permutation-invariant aggregation
over the bag~\cite{ilse2018}, in the same spirit as other set
architectures~\cite{zaheer2017,lee2019}.

Bags differ in size, so batched implementations store them in a dense array
of $n$ rows together with a validity mask or a length. The model is then
expected to satisfy two properties:
\begin{enumerate}
\item \textbf{Padding invariance.} Appending masked rows, with whatever
contents, does not change the pooled output.
\item \textbf{Permutation invariance.} Reordering the rows, with each row's
features, score, and mask moving together, does not change the output.
\end{enumerate}
These properties are easy to state and easy to break. Zeroing the features of
padded rows is not enough: a padded row with logit $0$ still contributes
$e^0=1$ to the softmax normalizer. Two real instances with values $2$ and $6$
and equal logits should pool to $4$. With $p$ zero-feature padded rows left
inside the softmax, they pool to $8/(2+p)$, so the output depends on the
batch's padding length. Masking logits with $-\infty$ fixes the normalizer
when every bag has a valid row and all padding is finite. It yields
\texttt{NaN} when a bag is empty or a padded feature is nonfinite. Masking with a
large negative constant gives exact zeros only through floating-point
underflow, and an empty bag then returns the mean of its padding.
Table~\ref{tab:idioms} summarizes these behaviours.

This note gives a precise definition of masked attention pooling with a
guarded normalizer and machine-checked proofs of its invariances. The
contributions are:
\begin{itemize}
\item Size-generic Lean~4 theorems: invariance under any bijection of the row
index type, invariance under appending any finite set of masked rows with
arbitrary real features and logits, irrelevance of masked values, and a zero
output for all-masked bags (Section~\ref{sec:proofs}).
\item A corollary covering any shared, instance-local scoring function, such as
a linear layer, an MLP, or a gated attention scorer.
\item A connection from the theorems to the code that runs. The JAX kernel is traced to a Jaxpr,
translated and independently imported by \texttt{jax-lean}~\cite{jaxlean},
and linked to the generic specification by checked certificates
(Section~\ref{sec:pipeline}).
\item A precise statement of the trusted base and of what the proof does not
cover (Section~\ref{sec:limits}).
\end{itemize}
The mathematics is elementary: a bijection reindexes a finite sum, and terms
with mass zero vanish. The point is to make the specification exact, including
the empty-bag convention and the requirement that masked rows lose their mass,
and to check that a specific JAX kernel meets it.

\section{Background}

\paragraph{Attention MIL pooling.}
Given instance embeddings $x_1,\dots,x_n\in\R^d$, attention pooling computes
logits $s_i=f(x_i)$ with a shared scorer $f$, weights
$a_i=e^{s_i}/\sum_k e^{s_k}$, and the bag embedding $\sum_i a_i x_i$. Ilse et
al.\ use $f(x)=w^\top\tanh(Vx)$ and a gated variant~\cite{ilse2018}. Our
results do not depend on the form of $f$, only on its being applied to each
instance separately.

\paragraph{JAX and Jaxprs.}
JAX~\cite{jax2018} traces a Python function at fixed input shapes and dtypes
into a Jaxpr, a typed first-order intermediate representation of primitive
operations. This trace is what JAX's transformations and compiler consume.

\paragraph{Lean, mathlib, and \texttt{jax-lean}.}
Lean~4~\cite{lean4} is an interactive theorem prover with a small trusted
kernel, and mathlib~\cite{mathlib2020} supplies the real numbers, the
exponential function, and finite sums over \lean{Fintype}s. The
\texttt{jax-lean} library~\cite{jaxlean} translates a Jaxpr into two Lean
artifacts: a readable tensor function, and a separately imported, typed
syntax tree of the Jaxpr together with an evaluator that gives it ideal
real-arithmetic semantics. For each function it emits a \emph{translation
certificate}, a Lean theorem stating that evaluating the imported Jaxpr agrees
with the readable function on every input of that shape. Lean checks this
theorem, so the readable function does not have to be trusted. The importer,
which decides what the original Jaxpr means in the Lean syntax, does.

\section{Masked attention pooling}\label{sec:def}

Let $I$ be a finite type of row indices and $d\in\mathbb{N}$. A padded bag is
given by features $x:I\to\R^d$, logits $s:I\to\R$, and a mask
$m:I\to\{\mathrm{true},\mathrm{false}\}$. The mask may be arbitrary, including
invalid rows interleaved with valid ones; a length-$\ell$ prefix mask is the
special case $m_i=[\,i<\ell\,]$.

\begin{definition}[Masked attention pooling]\label{def:pool}
The \emph{mass} and \emph{masked feature} of row $i$ are
\[
w_i=\begin{cases}e^{s_i}, & m_i,\\ 0, & \neg m_i,\end{cases}
\qquad
v_i=\begin{cases}x_i, & m_i,\\ 0, & \neg m_i.\end{cases}
\]
With $Z=\sum_{i\in I}w_i$ and $D=Z$ if $Z>0$ and $D=1$ otherwise,
\[
\Pool(x,s,m)=\frac{1}{D}\sum_{i\in I}w_i\,v_i\;\in\R^d .
\]
\end{definition}

If some row is valid then $Z>0$, and the coefficients $w_i/Z$ form a
probability distribution supported on the valid rows, so this is the usual
attention average over the real instances. If no row is valid, or $I$ is
empty, the numerator is the zero vector and the output is $0$. (These two
remarks follow directly from $e^{s}>0$; only the second is formalized, as
Proposition~\ref{prop:allmasked}.) The guard on $D$ fixes the value on empty
bags and avoids $0/0$. The Lean definition is a direct transcription:

\begin{lstlisting}[language=Lean]
noncomputable def mass (logits : ι → ℝ) (mask : ι → Bool) (i : ι) : ℝ :=
  if mask i then Real.exp (logits i) else 0

noncomputable def pool [Fintype ι] (values : ι → Fin d → ℝ)
    (logits : ι → ℝ) (mask : ι → Bool) : Fin d → ℝ :=
  let total := ∑ i, mass logits mask i
  fun j => (∑ i, mass logits mask i * (if mask i then values i j else 0)) /
    (if total > 0 then total else 1)
\end{lstlisting}

The certified JAX kernel follows the same structure. The \texttt{where} on the
logits keeps a masked $+\infty$ out of the exponential, and the \texttt{where}
on the features keeps $0\cdot\texttt{NaN}$ out of the product:

\begin{lstlisting}[language=Python]
def masked_pool(values, logits, mask):
    safe_logits = jnp.where(mask, logits, 0.0)
    safe_values = jnp.where(mask[:, None], values, 0.0)
    weights = jnp.where(mask, jnp.exp(safe_logits), 0.0)
    total = jnp.sum(weights)
    denominator = jnp.where(total > 0.0, total, 1.0)
    return jnp.sum(weights[:, None] * safe_values, axis=0) / denominator
\end{lstlisting}

\begin{table}[t]
\centering
\small
\caption{Behaviour of common masking idioms in JAX (float32). The bag is the
three-instance example of Section~\ref{sec:validation}, whose correct output is
$(2.0572,\,1.3532)$. ``Zero pad'': two appended rows with zero features and zero
logits. ``NaN pad'': the same rows with \texttt{NaN} features. ``Empty'': two
rows, both masked, with features $(1,2)$ and $(3,4)$. The last column asks
whether the idiom is exactly padding-invariant in real arithmetic.}
\label{tab:idioms}
\begin{tabularx}{\linewidth}{@{}>{\raggedright\arraybackslash}Xcccc@{}}
\toprule
Idiom & Zero pad & NaN pad & Empty & Exact in $\R$\\
\midrule
Zero padded features only & $(1.3563,0.8921)$ & $(1.3563,0.8921)$ & $(0,0)$ & no\\
Logits $\to-\infty$, features unmasked & correct & \texttt{NaN} & \texttt{NaN} & n/a\\
Logits $\to-10^{9}$, features unmasked & correct & \texttt{NaN} & $(2,3)$ & no\\
Softmax of $-\infty$-masked logits, features zeroed & correct & correct & \texttt{NaN} & n/a\\
Definition~\ref{def:pool} (\texttt{masked\_pool}) & correct & correct & $(0,0)$ & yes\\
\bottomrule
\end{tabularx}
\end{table}

Table~\ref{tab:idioms} compares this kernel with four common alternatives,
executed in JAX. ``n/a'' marks idioms that use $-\infty$, which has no
counterpart in $\R$. The $-10^9$ idiom is correct in float32 because
$e^{-10^9}$ underflows to exactly zero. In real arithmetic the padded rows
keep a small positive mass, so exact invariance cannot be proved for it.

\section{Theorems}\label{sec:proofs}

All statements below are proved in \lean{MaskedAttention/Basic.lean} for
arbitrary finite types and arbitrary $d$. No bag size, feature value, or
shape is enumerated. We give the informal argument and the Lean statement for
each.

\begin{theorem}[Reindexing]\label{thm:reindex}
Let $K$ and $I$ be finite and $e:K\simeq I$ a bijection. Then
$\Pool(x\circ e,\;s\circ e,\;m\circ e)=\Pool(x,s,m)$.
\end{theorem}

\begin{proof}
Under the reindexing, $w'_k=w_{e(k)}$ and $v'_k=v_{e(k)}$. For any $g:I\to\R$,
$\sum_{k\in K}g(e(k))=\sum_{i\in I}g(i)$ because $e$ is a bijection. Applying
this to $g=w$ gives $Z'=Z$, so $D'=D$. Applying it to $g_i=w_i v_{i,j}$ for each
coordinate $j$ gives equal numerators.
\end{proof}

\begin{lstlisting}[language=Lean]
theorem pool_reindex [Fintype ι] [Fintype κ] (e : κ ≃ ι)
    (values : ι → Fin d → ℝ) (logits : ι → ℝ) (mask : ι → Bool) :
    pool (values ∘ e) (logits ∘ e) (mask ∘ e) = pool values logits mask := by
  funext j
  unfold pool mass
  simp only [Function.comp_apply]
  rw [e.sum_comp (fun i => if mask i then Real.exp (logits i) else 0)]
  rw [e.sum_comp (fun i => (if mask i then Real.exp (logits i) else 0) *
    (if mask i then values i j else 0))]
\end{lstlisting}

Taking $K=I$ gives permutation invariance. The theorem is stated for distinct
$K$ and $I$ because the concrete padding result in Section~\ref{sec:pipeline}
needs to identify $\mathrm{Fin}\,5$ with $\mathrm{Fin}\,3\oplus\mathrm{Fin}\,2$.

\begin{theorem}[Padding]\label{thm:padding}
Let $P$ be any finite type, and let $p:P\to\R^d$ and $t:P\to\R$ be arbitrary.
Write $x\sqcup p$ for the function on $I\oplus P$ that agrees with $x$ on $I$
and with $p$ on $P$, and likewise for the other arguments. Then
\[
\Pool\bigl(x\sqcup p,\; s\sqcup t,\; m\sqcup\mathrm{false}\bigr)=\Pool(x,s,m).
\]
\end{theorem}

\begin{proof}
A sum over $I\oplus P$ splits into a sum over $I$ plus a sum over $P$. Every
padded row is masked, so its mass and masked feature are $0$ regardless of
$p$ and $t$. Both the normalizer and the numerator therefore equal their
unpadded values.
\end{proof}

\begin{lstlisting}[language=Lean]
theorem pool_padding [Fintype ι] [Fintype κ]
    (values : ι → Fin d → ℝ) (logits : ι → ℝ) (mask : ι → Bool)
    (padding : κ → Fin d → ℝ) (paddingLogits : κ → ℝ) :
    pool (Sum.elim values padding) (Sum.elim logits paddingLogits)
      (Sum.elim mask (fun _ => false)) = pool values logits mask := by
  funext j
  simp [pool, mass, Fintype.sum_sum_type]
  rfl
\end{lstlisting}

The padded features and logits are arbitrary reals; they need not be zero.
Rows already masked inside $I$ are allowed, and $P$ may be empty. Combining
Theorems~\ref{thm:reindex} and~\ref{thm:padding}, any bijection of the padded
index set $I\oplus P$ that moves the mask with the rows also leaves the output
unchanged. Padded rows may be interleaved anywhere.

\begin{proposition}[Masked values are irrelevant]\label{prop:congr}
If $x_i=x'_i$ and $s_i=s'_i$ for every $i$ with $m_i$, then
$\Pool(x,s,m)=\Pool(x',s',m)$.
\end{proposition}

\begin{proposition}[All-masked bags]\label{prop:allmasked}
$\Pool(x,s,\mathrm{false})=0$ for all $x$ and $s$, including when $I$ is empty.
\end{proposition}

Both are proved by case analysis on the mask (\lean{pool_masked_congr},
\lean{pool_all_masked}).

\begin{corollary}[Instance-local scorers]\label{cor:scorer}
Let $f:\R^d\to\R$ be any function and $e:I\simeq I$ a permutation. Then
\[
\Pool\bigl(x\circ e,\;f\circ x\circ e,\;m\circ e\bigr)=\Pool(x,\;f\circ x,\;m).
\]
\end{corollary}

In Lean this is \lean{attention_permutation}. The logit $f(x_{e(i)})$
travels with its instance, so it is Theorem~\ref{thm:reindex} with $s=f\circ x$.
No assumption is made about $f$: it can be linear, an MLP, the
$\tanh$ scorer of~\cite{ilse2018}, or its gated variant. The corollary does
not cover scorers that use position or other instances. A positional encoding
or a self-attention encoder before pooling needs its own equivariance
argument.

\section{From the JAX kernel to the theorems}\label{sec:pipeline}

The theorems of Section~\ref{sec:proofs} are about a Lean definition. To
connect them with the kernel above, we trace \texttt{masked\_pool} with JAX at
shapes $[3,2]$ and $[5,2]$ (features $[n,2]$, logits and mask $[n]$) and pass
each Jaxpr to \texttt{jax-lean}'s \lean{certify_module}. The generated files
are reproducible: a check script regenerates them and requires byte equality.
Figure~\ref{fig:pipeline} shows the chain of results.

\begin{figure}[t]
\centering
\begin{tikzpicture}[
  node distance=16mm and 16mm,
  box/.style={draw, rounded corners=2pt, align=center, font=\small,
    minimum height=10mm, text width=30mm, fill=white},
  trusted/.style={box, fill=gray!12, dashed},
  arr/.style={-{Stealth[length=2mm]}, thick},
  lab/.style={font=\scriptsize, align=center, fill=white, inner sep=1.5pt}]
\node[trusted] (py) {JAX kernel\\\texttt{masked\_pool}};
\node[trusted, right=of py] (jaxpr) {Jaxpr at\\$[3,2]$ and $[5,2]$};
\node[trusted, right=of jaxpr] (ir) {Imported IR\\\lean{pool_ir}};
\node[box, below=of ir] (fn) {Readable function\\\lean{Generated3.pool}};
\node[box, left=of fn] (spec) {Generic \lean{pool}\\(Definition~\ref{def:pool})};
\node[box, left=of spec] (thm) {Invariance theorems\\\ref{thm:reindex} and \ref{thm:padding}};
\draw[arr] (py) -- node[lab, above=1pt] {trace} (jaxpr);
\draw[arr] (jaxpr) -- node[lab, above=1pt] {import} (ir);
\draw[arr, {Stealth[length=2mm]}-{Stealth[length=2mm]}] (ir) -- node[lab, right=2pt] {translation\\certificate} (fn);
\draw[arr, {Stealth[length=2mm]}-{Stealth[length=2mm]}] (fn) -- node[lab, above=1pt] {\lean{_spec}} (spec);
\draw[arr] (thm) -- node[lab, above=1pt] {apply} (spec);
\end{tikzpicture}
\caption{Chain of results. Shaded, dashed boxes and the arrows between them
are trusted: JAX tracing and the \texttt{jax-lean} importer that encodes the
Jaxpr as Lean syntax. The vertical and lower arrows are theorems checked by
Lean. The \lean{certified_*} theorems compose them to state both invariances
directly about the imported IR.}
\label{fig:pipeline}
\end{figure}
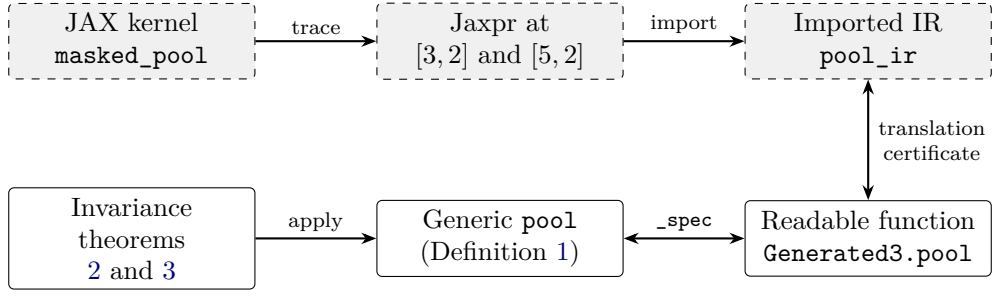

\paragraph{Translation certificates.}
For each shape, \texttt{jax-lean} emits a readable function such as
\lean{Generated3.pool} and an imported program \lean{Generated3.pool_ir},
plus helper functions for each \texttt{jnp.where}. The theorem
\lean{pool_translation_correct} states that evaluating \lean{pool_ir} with
\lean{Jaxpr.Program.eval} gives the readable
function on every real tensor input of that shape. It is proved by
\texttt{jax-lean}'s certificate tactic from local evaluator lemmas.

\paragraph{Specification lemmas.}
\lean{generated3_spec} and \lean{generated5_spec} prove that each readable
function equals the generic \lean{pool}, read at the corresponding indices.
These are the steps that check the kernel's content: the three mask branches,
the leading-axis reductions, the comparison \mbox{\lean{total > 0}}, and the guarded
division. A kernel that zeroed only the features would fail here, because its
normalizer would include $e^{s_i}$ for masked rows.

\paragraph{Certified invariances.}
The final theorems are stated about the imported program:
\begin{itemize}
\item \lean{certified_permutation}: for every $e:\mathrm{Fin}\,3\simeq\mathrm{Fin}\,3$
and all real $x\in\R^{3\times2}$, $s\in\R^3$, $m\in\{\mathrm{t},\mathrm{f}\}^3$,
evaluating the $[3,2]$ program on the row-permuted inputs gives the same
result as on the original inputs.
\item \lean{certified_padding}: for all such $x,s,m$ and arbitrary
$p\in\R^{2\times2}$, $t\in\R^2$, evaluating the $[5,2]$ program on the
inputs with two masked rows appended gives the same result as the $[3,2]$
program on the original inputs.
\end{itemize}
The padding statement relates two separately traced programs. Its proof
rewrites both sides through their certificates and specification lemmas,
reindexes $\mathrm{Fin}\,5$ as $\mathrm{Fin}\,3\oplus\mathrm{Fin}\,2$ with
Theorem~\ref{thm:reindex}, and applies Theorem~\ref{thm:padding}. The generic
theorems hold for every size, but a Jaxpr has fixed shapes, so certificates
exist only for the shapes that were traced.

\section{Validation}\label{sec:validation}

The development was built with Lean~4.33.0, the mathlib release pinned by
\texttt{jax-lean} at revision \texttt{570aed6}, and JAX/jaxlib~0.8.0 on a CPU
device. A single script regenerates and compares the generated Lean files,
runs the Python tests, builds the Lean library, and audits seven theorems:
the five theorems of Section~\ref{sec:proofs} and the two certified
invariances. Each depends only on \lean{propext}, \lean{Classical.choice},
and \lean{Quot.sound}, the standard axioms used throughout mathlib. None
depends on \lean{sorryAx} or on a project-specific axiom.
Table~\ref{tab:size} gives the size of the development.

\begin{table}[t]
\centering
\small
\caption{Size of the development (lines, including comments).}
\label{tab:size}
\begin{tabular}{@{}lrl@{}}
\toprule
File & Lines & Content\\
\midrule
\lean{Basic.lean} & 79 & Definition~\ref{def:pool}, Theorems~\ref{thm:reindex}--\ref{thm:padding}, Props.~\ref{prop:congr}--\ref{prop:allmasked}, Cor.~\ref{cor:scorer}\\
\lean{Connection.lean} & 79 & specification lemmas, concrete and certified invariances\\
\lean{Generated3.lean} & 200 & emitted function, imported IR, certificates for $[3,2]$\\
\lean{Generated5.lean} & 200 & the same for $[5,2]$\\
\bottomrule
\end{tabular}
\end{table}

The theorems concern real arithmetic, so we also test execution separately.
Six JAX tests compare the kernel and a numerically stabilized variant
(Section~\ref{sec:limits}) against an independent NumPy formula. They cover
bag sizes $1$, $3$, and $17$, random permutations, $0$, $1$, and $11$ appended
rows filled with \texttt{NaN} features and $+\infty$ logits, all-masked and
empty bags, eager and JIT execution, and logits near $1000$. As a worked
example, take
\[
X=\begin{pmatrix}1&2\\3&4\\2&-1\end{pmatrix},\qquad
a=\begin{pmatrix}0.2\\-0.1\end{pmatrix},\qquad s=Xa=\begin{pmatrix}0\\0.2\\0.5\end{pmatrix}.
\]
With all three rows valid, the pooled output is $(2.057208,\,1.353158)$ in
both float64 and float32. Appending two rows of \texttt{NaN} or shuffling the
rows gives the same float32 output. Tests use tolerances because reordering a
floating-point sum can change its last bits~\cite{goldberg1991}. They are
evidence about execution and do not replace the proofs.

\section{Scope and limitations}\label{sec:limits}

\paragraph{Real arithmetic.}
The theorems are exact identities over $\R$. They do not model IEEE-754
rounding, overflow, \texttt{NaN}, or infinities, and they say nothing about
XLA's compiled code. In floating point the permuted and padded outputs agree
only up to rounding, which is why the tests use tolerances.

\paragraph{Trusted components.}
The Lean kernel and mathlib's definitions are trusted, as in any Lean
development. On the JAX side, JAX's tracer is trusted to produce the Jaxpr that
is executed, and the \texttt{jax-lean} importer is trusted to encode that
Jaxpr faithfully as Lean syntax with real-number semantics, with dtype
conversion erased~\cite{jaxlean}. The readable generated functions are not
trusted: they are linked to the imported program by checked certificates.

\paragraph{Fixed shapes.}
The generic theorems quantify over every finite index type and feature
dimension, but only the $[3,2]$ and $[5,2]$ traces are certified. Another
shape requires tracing and certifying it again, which the generator script
does mechanically.

\paragraph{The stabilized kernel is not certified.}
In practice the exponentials are shifted by the maximum valid logit to avoid
overflow. The repository's \texttt{stable\_masked\_pool} does this and is
covered by the numerical tests. Its own Jaxpr is not certified in this
version: the masked maximum uses $-\infty$ as an initial value and identity,
which falls outside the finite-real fragment that \texttt{jax-lean} models.
Mathematically the shift cancels between numerator and normalizer, so it does
not change Definition~\ref{def:pool} for finite logits. A certified version
would need a finite formulation of the masked maximum or an extended-real
semantics.

\paragraph{Architecture scope.}
Corollary~\ref{cor:scorer} covers scorers applied to each instance
separately. Encoders that mix instances before pooling, such as a
transformer over instances, positional encodings, or clustering losses, need
their own equivariance proofs. The results are structural. They make no
claim about accuracy, calibration, or clinical validity.

\section{Related work}

Selsam et al.\ wrote a learning system and its correctness theorem together
in Lean and verified that its sampled gradients are unbiased~\cite{selsam2017}.
Our goal is narrower and works after the fact: we prove one structural
property of an existing JAX kernel and connect it to the traced program
through \texttt{jax-lean}'s translation certificates~\cite{jaxlean}. That
library's examples include softmax normalization, masking, and transformer
components. The present development follows its transformer proofs and adds a
mask-aware pooling specification, padding across separately traced shapes,
and the instance-local scorer corollary.

Permutation-invariant set functions are a standard design principle for
learning on sets~\cite{zaheer2017,lee2019}, and attention MIL pooling is
introduced as one~\cite{ilse2018}. Those works establish invariance of the
mathematical operator on unpadded sets. We address the padded, masked
implementation and give the empty-bag convention explicitly. Verification of
trained neural networks usually concerns input--output properties of
particular weights, such as robustness. Our results hold for all weights and
all scorers and concern the aggregation layer only.

\section{Conclusion}

Masked attention pooling is padding- and permutation-invariant if masked rows
contribute zero mass to the normalizer as well as the numerator, and the
normalizer is guarded for empty bags. We stated this precisely, proved it in
Lean for every finite bag and feature dimension, and connected the proofs to
the Jaxpr traced from a JAX kernel for two concrete shapes. The proofs are
short, and the main value is the exact specification: it rules out the
feature-only masking error and fixes behaviour on empty bags and nonfinite
padding. Natural extensions are a certified max-shifted kernel and proofs for
encoders placed before pooling.

\section*{Code availability}

Source, generated files, tests, and the validation log are at
\url{https://github.com/y2x0/masked-attention-lean}, commit \texttt{97978b6}.
The \texttt{jax-lean} dependency is pinned as a submodule at revision\\
{\small\texttt{570aed69049f8e2931df9f55175313a340bcb456}}.\\
The full check runs with
\texttt{./scripts/check.sh}.

\section*{Acknowledgments}

Claude (Anthropic, Claude Opus 5.5) helped write and revise this manuscript
from the project's source, documentation, and validation record. The authors
reviewed the text and take full responsibility for its content.

\end{document}